\def\year{2020}\relax
\documentclass[letterpaper]{article}
\usepackage{aaai20}
\usepackage{times}
\usepackage{helvet}
\usepackage{courier}
\usepackage{url}
\usepackage{graphicx}
\usepackage{subcaption}
\usepackage{soul}
\usepackage{booktabs}
\usepackage{latexsym,amssymb,amsthm}
\usepackage{xspace}
\usepackage{enumerate}
\usepackage[defblank]{paralist}
\usepackage{microtype}
\usepackage{comment}
\usepackage{todonotes}
\usepackage{tabularx}
\usepackage{booktabs}
\usepackage{multirow}
\usepackage{MnSymbol}
\usepackage[linesnumbered,ruled,vlined]{algorithm2e}

\newcommand{\qlearn}{\text{Q-learning}}
\newcommand{\rmax}{$R_{max}$}

\newcommand{\lstar}{$L^{*}$}
\newcommand{\edsm}{\text{EDSM}}

\newcommand{\commentout}[1]{}

\newcommand{\A}{\mathcal{A}}

\newcommand{\M}{\mathcal{M}} 
\newcommand{\LTLf}{{\sc ltl}$_f$\xspace}

\newcommand{\LDLf}{{\sc ldl}$_f$\xspace}

\newcommand{\Nat}{{\rm I\kern-.23em N}}

\newtheorem{theorem}{Theorem}

\title{Reinforcement Learning with Non-Markovian Rewards}

\author{
Maor Gaon and Ronen I. Brafman\\
Ben-Gurion University of the Negev, Beer Sheva, Israel\\
maorga@post.bgu.ac.il\\
brafman@cs.bgu.ac.il}

\begin{document}

\maketitle

\begin{abstract}
  The standard RL world model is that of a Markov Decision Process (MDP).
  A basic premise of MDPs is that the rewards depend on the last state and
  action only. Yet, many real-world rewards are non-Markovian. For example, 
  a reward  for bringing coffee only if requested earlier and not yet served, is non-Markovian
  if the state only records current requests and deliveries. Past work
  considered the problem of modeling and solving MDPs with non-Markovian rewards (NMR),
  but we know of no principled approaches for RL with NMR. 
  Here, we address the problem of policy learning from experience with such rewards.
   We describe and evaluate empirically four combinations of the classical RL algorithm
  Q-learning and R-max with automata learning algorithms to obtain new RL
  algorithms for domains with NMR. We also prove that 
  some of these variants converge to an optimal policy in the limit.
\end{abstract}

\section{Introduction}
  The standard reinforcement learning (RL) world model is that of a Markov Decision Process (MDP).
  In MDPs, rewards depend on the last state and
  action only. But a reward given
  for bringing coffee only if one was requested earlier and was not served yet, is non-Markovian
  if the state only records current requests and deliveries; and so is a 
  reward for food served only if one's hands  were washed earlier, and no activity that made
  them unclean occurred in between. An RL agent that attempts to learn in such domains without
  realizing that the rewards are non-Markovian will display sub-optimal behavior. For example, depending
  on how often coffee is requested and how costly making coffee is, the agent may learn
  never to make coffee or always to make coffee -- regardless of the user's requests.
  
  We can address non-Markovian rewards (NMRs) 
  by augmenting the state with information that makes the new model Markovian.
  Except for pathological cases, this is always possible. 
  For example, the robot can maintain a variable indicating whether
  an order for coffee was received and not served. Similarly, it can keep track of whether the user's hands are
  clean. Indeed, early work on this topic described various methods for efficiently augmenting the
  state to handle NMRs~\cite{BBG96,Thiebaux06}. More recent work leveraged our better 
  understanding of the relationship between NMRs specified using temporal and dynamic logics and automata to
  provide even better methods~\cite{CamachoCSM17b,BDP18}.  
  Yet, such model transformations are performed offline with complete knowledge of the model and
  the reward structure, and are not useful for a reinforcement learning agent equipped with some fixed state model
  and state variables. 
  
  In this paper, we combine our recent, better understanding of NMRs and automata~\cite{CamachoCSM17b,BDP18}
  with  automata learning algorithms~\cite{Angluin87,GILearning} and  classical RL algorithms~\cite{WatDay,Rmax}
  to address RL with NMRs. We describe algorithms that allow a learning agent with no prior knowledge of the reward 
  conditions to learn how to augment its state space with new features w.r.t.~which rewards are Markovian.
  The added state variables are not directly observable.  However, while regular state variables 
  are affected stochastically by actions, for the rich class of NMRs discussed in this and past work, these state variables
  are affected deterministically by the actions. 
  Hence, in practice, their value can be tracked and updated by the agent.
  
  Like standard RL algorithms, our algorithms explore the state space, learning a transition model or estimating the $Q$ function.
  They also maintain a record of executions (traces) leading to a reward
  as well as traces not leading to a reward. 
  These traces are fed into an off-the-shelf automata learning algorithm, yielding  a deterministic finite-state automata (DFA). 
  The state of this DFA is added as a new state feature. Since this feature is a regular function of the agent's history,
  the agent can constantly update it and maintain its correct value.
  Given the augmented state-space, the reward is Markovian, and is always obtained when the DFA reaches an accepting state.
  The algorithm now adjusts its model to take account of the new state space, and can continue to explore or to exploit.

  We describe four algorithms of this type obtained by combining 
  Q-learning~\cite{WatDay} and \rmax~\cite{Rmax} with two automata learning
  algorithms, the online \lstar~\cite{Angluin87} and the offline \edsm~\cite{EDSM}. We show that
  under certain assumptions, the algorithms that use \lstar\ will converge to an
  optimal policy in the limit, and we evaluate them empirically.
\section{Background}
  We briefly discuss MDPs, RL, NMRs and automata learning.

\subsection{MDPs, RL and NMR}
  RL is usually formulated in the context of a Markov Decision Process (MDP)
  $\M = \langle S,A,Tr,R\rangle$. 
  $S$ is the set of states, $A$ is the set of actions,  $Tr:S\times A\rightarrow \pi(S)$ is the transition
  function that returns for every state $s$ and action $a$ a distribution over the next state. 
  $R:S\times A \rightarrow \mathbb{R}$ is the reward function that
  specifies the real-valued reward received by the agent when applying
  action $a$ in state $s$.

  A solution to an MDP is a {\em policy} that assigns an action to each
  state, possibly conditioned on past states and actions. 
  The {\em value\/} of policy $\rho$ at $s$, $v^{\rho}(s)$, is the expected sum of
  (possibly discounted) rewards when starting at $s$ and selecting
  actions based on $\rho$. Focusing on the infinite horizon, discounted reward case, we
  know that every MDP has an {\em optimal} policy, $\rho^*$ that maximizes the expected discounted sum
  of rewards for every starting state $s\in S$, and that an optimal policy that is stationary and deterministic  
  $\rho^*: S \rightarrow A$ exists~\cite{Puterman}.  The value function obeys Bellman's optimality condition~\cite{Bellman}
  $$v^*(s) = \max_{a\in A} R(s,a) + \gamma \sum_{s'\in S} Tr(s,a,s')v^*(s')$$ The  state-action value function,
  or $Q$-function, is defined as $$Q(s,a) = R(s,a) + \gamma \sum_{s'\in S} Tr(s,a,s')v^*(s')$$
  hence $$v^*(s) = \max_{a\in A} Q(s,a)$$ and $$Q(s,a) = R(s,a) + \gamma \sum_{s'\in S} Tr(s,a,s')\max_{a\in A} Q(s,a)$$

  RL algorithms attempt to find or approximate $\rho^*$ through experience, without prior knowledge of $\M$.
  There are various methods of doing this~\cite{Sutton2}, two of which we discuss later.

  Non-Markovian Reward Decision Processes (NMRDPs)~\cite{BBG96}  extend MDPs to allow for
  rewards that depend on the entire history. It has been observed by much past work, and more recently
  in the area of robotics~\cite{LittmanTFIWM17} that many natural rewards are non-Markovian.
  A general NMR is a function from $(S\times A)^*$ to $\mathbb{R}$.
  However, this definition is too complex to be of practical use because of its infinite structure. 
  For this reason, past work focused on properties of histories that are finitely describable, 
  and in particular, logical languages that are evaluated w.r.t.~finite traces of states and actions,
  representing the agent's history. Recent work has focused on \LTLf\  and \LDLf~ \cite{DegVa13}, in particular. 
  There are various technical advantages to using these languages to specify NMRs (see,~\cite{BDP18} for more details).
  Here, due to space constraints and to make the description more accessible to RL researchers less familiar with these logics,
  we exploit the fact that the expressive power of \LDLf\ is equivalent to that of regular languages, with which most computer scientists are familiar. 
  Thus, we focus on rewards that can be described as regular functions of the agent's state-action history.

\subsection{Automata Learning}
  A {\em deterministic finite state automaton} (DFA) is a tuple $\A=\langle S,\Sigma,\delta,s_0,F\rangle$.
  $S$ is the set of states, $\Sigma$ is the alphabet, $\delta:S\times\Sigma\rightarrow S$ is a deterministic transition function,
  $s_0\in S$ is the initial state, and $F\subseteq S$ are the final/accepting states. 
  It is essentially a deterministic MDP with actions $\Sigma$. 

  Given the equivalence between regular expressions and DFAs, if rewards are 
  regular functions of the agent's state-action history, we can try to learn an automaton that accepts
  exactly those histories that led to a reward. Here we briefly describe two such classical automata learning algorithms.

  \lstar\ is an interactive DFA learning algorithm with an external oracle (a.k.a.~the teacher) that provides answers to membership and equivalence queries. 
  In a membership query the   teacher is asked whether a string belongs to the language. In an equivalence query, it is given a DFA. 
  If this is the correct DFA, the teacher notifies the learner. Otherwise,
  it provides a counter-example -- a string that is not correctly classified by the DFA. This algorithm requires
  a polynomial number of queries in the size of the minimal DFA.

  Let $\Sigma$ be our alphabet, and $L$ the target language. At each step of the algorithm, the learner maintains: 
  a set $Q\subset \Sigma^*$ of {\em access} words and a set $T\subset \Sigma^*$ of {\em test words}. 
  $v,w\in\Sigma^*$ are {\em T-equivalent} (denoted $v\equiv_T w$) if   $\forall u\in T$, $vu\in L \Leftrightarrow wu\in L$.
  $Q$ and $T$ are {\em separable} if no two distinct words in $Q$ are $T$-equivalent.
  $Q$ and $T$ are {\em closed} if for every $q\in Q$ and $a\in\Sigma$ there is some $q'\in Q$ such
  that $q\cdot a\equiv_T q'$.

  If $(Q,T)$ is separable and closed and we know for each $v\in Q\cup T$ whether it is in $L$,
  then we can define a hypothesis automaton $\mathcal{H}$
  as follows: the states of $\mathcal{H}$ are $Q$. $\epsilon$ is the initial state.
  If $\mathcal{H}$ is in state $q\in Q$ and reads $a\in\Sigma$, it transitions to the unique $q'\in Q$ such that $q\cdot a\equiv_T q'$.
  Such a state exists because $(Q,T)$ is closed. It is unique by separability.
  The accepting states of $\mathcal{H}$ are those state $q\in Q$ such that $q\in L$.

  The algorithm proceeds as follows: (1) Initially, $Q:=T:=\{\epsilon\}$. They are clearly separable.
  (2) Repeatedly enlarge $Q$, maintaining its separableness, until $(Q,T)$ is separable and closed. 
  This is possible because, if $(Q,T)$ are not closed, there exists some $q\in Q$ and $a\in\Sigma$ such that
  no $q'\in Q$ is $T$-equivalent to $q\cdot a$. We can find them by membership queries, and add $q\cdot a$ to $Q$.
  (3) Compute the hypothesis automaton $\mathcal{H}$ and use it for an equivalence query. (4) If the answer is YES, terminate with
  success. (5) If the answer is NO with counter example $w$, let $|w|=n$. Using $\log n$ membership queries, one can find
  some $q\not\in Q$ and $t\in\Sigma^*$ such that $(Q\cup{q},T\cup{t})$ is separable. This is done by finding, using binary
  search, two adjacent prefixes of $w$, $w'$ and $w'\cdot a$ ($a\in \Sigma$) that yield words that are tagged differently
  by the language when we do the following process: (a) compute the state $q_i$ and $q_{i+1}$ we reach
  in $\mathcal{H}$ on  $w'$ and $w'\cdot a$; (b) look at the words obtained by replacing $w'$ and $w'\cdot a$ in
  $w$ with the words associated with $q_i$ and $q_{i+1}$. (c) Check their tags.
  (6) Goto 2.

  Unlike \lstar, the class of state merging algorithms~\cite{EDSM}, work off-line on a given set of positive and negative samples. First, 
  a Prefix Tree Acceptor (PTA) is constructed. A PTA is a tree-like DFA. 
  It is built by taking all the prefixes of words in the sample as states and constructing the smallest DFA on these states which 
  is a tree and is strongly consistent with the sample (i.e., accepts every positive sample and rejects every negative one). 

  State merging techniques iteratively consider an automaton and two of its states and aim to merge them. 
  States can be merged if they are compatible. A minimal requirement is that the union of prefixes associated with this
  state will not contain both a positive and a negative sample. This is not sufficient because certain merges imply other merges.
  For example, if we merge a state $s$ containing $\epsilon$ with a state containing $a$, this means that there is a self transition
  from $s$ on $a$. Therefore, this state must also be merged with the prefix $aa$, if it exists.

  Merging leads to generalization because we often merge states containing unlabelled prefixes with states containing
  labeled one. Thus, if $\epsilon$ is accepting and $a$ has no label, when we merge them, we make $a$ accepting.
  In fact, due to the example above, we must actually make $a^*$ accepting.

  The order by which pairs of states are considered for merging is the key difference between the different state merging algorithms. 
  Any wrong merge can affect the final DFA dramatically, and the earlier a wrong merge is performed the higher its impact on the final result. 
  State merging algorithms use various heuristics to select which states to merge. 
  They often perform the merges based on the sum of accepting/rejecting states already observed in the two candidates, 
  preferring candidates with large numbers.

\subsection{Related Work}
  A number of authors have recently emphasized the fact that many desired behaviors are non-Markovian.
  Littman~\cite{Littman15,LittmanTFIWM17} discussed the need for more elaborate reward specification for RL for robotics. 
  He considered scenarios where RL is used to learn a policy in an unknown world where the reward is not intrinsic to the world, 
  but is specified by the designer, usually in a high-level language. For this, he proposed the use of a temporal logic called GLTL. 
  Similarly,~\cite{Li+17} use truncated LTL as a reward specification language,
  and~\cite{CamachoCSM17,Toro+18} use \LTLf\ to specify desired complex behavior.
  Because temporal formulas are evaluated over an entire trace, it is difficult to guide the RL agent locally towards desirable behaviors. 
  \cite{CamachoCSM17} show how to to use the theory of \LTLf\ to use reward shaping to provide correct and useful feedback to the learning agent early on. 

  In all above papers, reward specification is part of the input. We are interested in RL when the reward model is unknown. 
  Some of the work on learning in partially observable environments can be viewed as indirectly addressing this problem.
  For example, classical work on Q-learning with memory~\cite{PMK01} maintains and updates an external memory. Thus,
  it essentially learns an extended state representation. But if we realize that the additional reward variables 
  are essentially states of an automaton, we can apply a more principled approach to learn these automata using state-of-the-art automata learning algorithms. 
  Not surprisingly, Angluin's famous automata learning algorithm~\cite{Angluin87} and early work on learning in unobservable environments (e.g.,~\cite{McCallum3}) 
  have a similar flavor of identifying states with certain suffixes of observations. Since then, automata learning has attracted much theoretical attention 
  because of the central role of automata in computational theory and in linguistics~\cite{LearningAutomata}. Unfortunately, exact learning of a target DFA 
  from an arbitrary set of labeled examples is hard~\cite{Gold78}, and  under standard cryptographic assumptions, it is not PAC-learnable~\cite{KV89}. 
  Thus, provably efficient automata learning is possible only if additional information is provided, as in Angluin's model that includes a teacher. However, 
  there is much work on practical learning algorithms, and in our experimental evaluation, 
  we use the FlexFringe implementation of the well known EDSM algorithm~\cite{EDSM,FlexFringe}. 

\section{Reinforcement Learning with NMRs}
  We present variants of Q-learning~\cite{WatDay} and \rmax~\cite{Rmax} that handle NMR. 
  In this paper we consider only tabular algorithms and  do not deal with function approximations.
  We assume that
  rewards are deterministic. In addition, we assume the agent can identify the reward type, i.e., 
  if two reward conditions have identical values, the agent can differentiate between them.
  While this assumption is natural in many cases (e.g., a robot may receive equal reward for bringing coffee or bringing tea and
  can distinguish between the two), it is made more for efficiency and convenience, and our approach can, in principle, simply
  learn the conjunctive condition.

\subsection{Learning with State Merging Algorithms}
  Algorithm~1 describes an algorithm that combines RL with a state merging algorithm (SMA).
  The agent repeatedly applies its favourite RL algorithm for a number of trials.
  Initially, it uses the original state space. The obtained traces (and implicitly, their prefixes) are stored in memory.
  Then, provided enough positive traces were collected, it calls the SMA with the stored traces
  and obtains a candidate automaton.
  $\A=\langle S_{\A},S_{\M},\delta,s_0,F\rangle$, where $S_{\A}$ is the set of automaton states and $S_{\M}$ -- the set of MDP states -- is the automaton's alphabet.
  (We assume for the sake of simplicity, that $S_{\M}$ also records the last applied action).

  Now, the MDP's state space is modified to reflect the automaton learned,
  replacing $S_{\M}$ with $S_{\M} \times S_{\A}$ (or $S_{\M}\times S_{\A_1}\cdots\times S_{\A_n}$ if there are $n$ NMRs). 
  Finally, in Line 10, 
  the information learned on the previous state space is used to initialize the
  values associated with the new state space.  (This update appears in both the SMA and \lstar\ variants, is
  specific to the RL algorithm and will be discussed later.) The algorithm now repeats the above process with
  the new state space.

  \begin{algorithm}[ht]
    \DontPrintSemicolon
	\caption{RL with NMR and EDSM} \label{alg:qlearn}
	Set $\A$ to some single-state automaton;\\
	\While {time remaining} {	
      Set $S=S_{\M}$ (MDP's state space);\\
      \For {$i=1$ to $c_{trials}$} {
        Execute one trial of RL algorithm on state space $S$;\\
        Store state-action-reward trace;\\
      }
      \If {there are more than $c_{pos}$ positive traces} {
        $\A$ = result of applying EDSM to the current sample;\\
        Update the state space $S = S_{\M}\times S_{\A}$, where $S_{\A}$ are the states of $\A$;\\
        Use estimates learned from $S_{\M}$ to generate initial estimates for updated $S_{\M}$;\\
      }
    }
    \label{alg:edsm}	
  \end{algorithm}

\subsection{Learning with \lstar}
  Algorithm~2 describes an \lstar-based\ algorithm for learning with NMRs.
  The \lstar\  algorithm provides the overall structure, and the RL algorithm learns in the background
  as it tries to answer queries. If \lstar\  asks a membership query, the agent tries to
  generate the corresponding trace, until it succeeds, at which point it can answer the query.
  This, of course, may take many trials, and in our implementation, if after at most $k$ attempts we do not 
  succeed in generating this trace, we tag the trace as negative.
  In each trial, when the trace deviates from the query trace, 
  the algorithm continues until the end of the trial using the current exploration policy.
  When \lstar\  asks an equivalence query, we check our stored
  traces for any counter-example. If none exists, the RL algorithm continues using the current automaton,
  until a counter-example is generated or until learning time ends. 

  All traces encountered are labeled and stored in memory, and the RL algorithm updates its data structure
  based on the stored samples using a state space which is the product of the MDP's $S_{\M}$ and the 
  current automaton's $S_{\A}$. 

  \begin{algorithm}[t]
    \SetAlgoLined
    Run the \lstar\ algorithm\;
    \While {time remaining} {
      Let $\A$ be the current hypothesis automaton for \lstar\;
      \If {\lstar\  asks a membership query $w$} {
        Done = {\em false}\;
        \While {Done = {\em false}} {
          \While {current trace is a strict prefix of $w$} {
            perform the next action in $w$\;
          }
          \eIf {current trace equals $w$} {
            return an answer to \lstar\ based on the label obtained\;
            Done = {\em true}\;
          }{  
            continue exploring using current exploration policy (e.g., $\epsilon$-greedy) until end of trial\;
            store the observed trace and its tag\;
            update the RL algorithm using the traces and state space $S_{\M}\times {S_{\A}}$\;
          }
        }
      }
      \If {\lstar\ asks an equivalence query} {
        \eIf {log contains a counter-example} {
          return the counter example to \lstar\;
        }{
          \Repeat {counter-example is found} {
            continue exploration recording observed traces and updating the RL algorithm\;
          }
          return the counter-example to \lstar\;
        }
      }
    }
    \caption{RL with NMR and \lstar} 
    \label{alg:lstar}	
  \end{algorithm}

  Because \lstar\ is guaranteed to converge to the correct automaton after a bounded (in fact, polynomial) number of queries,
  we can prove the following result. 

\begin{theorem}
  Assuming rewards are regular functions of the state-action history,
  the RL with NMR and \lstar\ algorithm will converge to an optimal policy in the limit with probability 1,
  if the underlying reinforcement algorithm converges to an optimal policy in the limit and selects
  every action in every state with some minimal probability $\epsilon$ for some $\epsilon>0$. 
\end{theorem}
\begin{proof}
  Since \lstar\ asks a finite and bounded number of queries, to prove this result, 
  it is sufficient to show that (a) membership queries will be answered in finite time with
  probability 1, and that (b) if the automaton hypothesized by \lstar\ is wrong, then 
  the equivallence query will be answered (correctly) in finite time with probability 1.
  For (a): observe that at each step there is probability $>c$, for some strictly positive value $c$
  (determined by the transition function) of making the desired transition when attempting to simulate trace $w$. 
  Since $w$ is finite, then with probability 1, we will succeed in finite time.
  For (b): let $w$ be the counter-example. We need to prove that it will be found with probability
  1 in finite time. The argument is the same as for (a), taking into account that $\epsilon$-greedy exploration
  implies a positive probability $>c'$, for some strictly positive $c'$, of taking each action in each state.
  Finally, once \lstar\ generates the correct automaton, the reward is Markovian w.r.t.~the new state space.
  Because the RL algorithm converges to an optimal policy with probability 1, so will our algorithm.
  Note that our algorithm may continue exploring, but its greedy policy will be the optimal one.
\end{proof}

\subsection{Learning Multiple Rewards}
  The descriptions above focus on a single reward automaton. To handle multiple automata, 
  a separate automaton is learned for each reward type. In EDSM, if multiple automata exist,
  we simply need to occasionally revise the state space when one of the automata changes due to new examples.
  In \lstar, we tried two schemes, which yielded identical results.
  In the first approach, we prioritize the automata and query higher priority automata first. 
  In the second approach, we interleave queries for different automata.  All else remains the same.

\subsection{The RL Algorithm}
  We now discuss the use of \qlearn\ and \rmax\ as the RL algorithm in Algorithms~1 \& 2.
  
\subsubsection{Q-learning}
  ~\cite{WatDay} is a model-free learning algorithm that maintains an estimate of the $Q$ function.
  The function is initialized arbitrarily, and is updated following each action $a$ as follows:
  $$Q(s,a) = (1-\alpha) Q(s,a) +\alpha\cdot (r+\gamma \max_{a'\in A} Q(s',a'))$$
  where $s$ is the state in which $a$ was applied, $r$ is the reward obtained, $s'$ is the resulting state, 
  $0< \alpha < 1$ is the learning rate, and  $0< \gamma < 1$ is the discount factor.

  To update the $Q$ values in Line 10 of Alg.~1 and Line~16 of Alg.~2,
  we initialize all $Q$ values to 0 except for states of the form $(s_{\M},s_{\A})$ where $s_{\A}$ is an accepting
  state of the automaton (i.e., one in which we will receive the NMR). Then, we do experience replay with the
  {\em stored} traces while tracking the automaton state. Thus, if we have a stored trial $s_0,a_1,r_1,s_1,a_2,\ldots, s_n$, 
  we can simulate the state of the automaton along this trace because the automaton is deterministic. 
  We obtain an updated trace: $(s_0,s_{\A,0}),a_1,r_1,(s_1,s_{\A,1} = \delta(s_{\A,0},s_1)),\ldots, (s_n,s_{\A,n} = \delta(s_{\A,n-1},s_n))$. 
  Because our automaton is consistent with past traces (with which it was trained), 
  we know that an accepting state $s_{\A,i}\in F$ will be reached only where the NMR was obtained. 

  Exploration is important for the convergence of RL algorithms. With NMRs, it is not enough to explore state-action pairs,
  but one must also explore traces, otherwise it may not encounter some potential rewards. 
  In \qlearn, once a positive NMR is obtained, it may reinforce the actions that were executed along the trajectory.
  In the next iteration, the algorithm will tend to apply these same actions,
  leading it to remain in the original trajectory, and to obtain the same reward.
  On the positive side because the reward is non-Markovian, we do want to strengthen the entire path.
  On the negative side, this prevents us from collecting good data for the automata learning algorithm, and 
  the learning is done by associating actions to state -- with no higher-level view of the path. 
  Moreover, any stochastic effect will throw us off the path. The algorithm may try to return to the rewarded path,
  but the deviation may already rule-out the desired NMR.

  If the domain is noisy, then there is a better chance of not repeating the same trajectory even if the same policy is used, 
  and this helps us explore alternative traces. Otherwise, the
  standard solution in regular RL is to introduce either a strong bias for optimism, or use an $\epsilon$-greedy policy.
  When $\epsilon$ is large enough, this also induces some trace exploration.

\subsubsection{$\mathbf{R_{max}}$}
  ~\cite{Rmax} is a model-based algorithm that estimates the transition  function based on the observed empirical distribution, 
  and also learns the reward function from observation. This algorithm has a strong exploration bias obtained by initially assuming 
  that every state and action that has not been sampled enough will result in a transition to a fictitious state in which it constantly obtains the maximal reward. 
  The algorithm follows the optimal policy based on its current model. 
  When it collects enough data about a transition on a state and an action, it updates the model, and recomputes its policy. 
  A key parameter of the algorithm is $K$ -- the number of times that an action $a$ must be performed in a state $s$ to mark the pair $(s,a)$ as {\em known}. 
  Once $(s,a)$ is known, $tr(s,a)$ is updated to reflect the empirical distribution. Under the assumption that rewards are deterministic, 
  $r(s,a)$ is updated the first time it is observed. 

  When an automaton $\A$ is learned, the old transition function $Tr_{old}:S_{\M}\times A \rightarrow \pi(S_{\M})$ 
  and reward function $R_{old}:S_{\M}\times A \rightarrow \mathbb{R}$ are replaced by new functions:
  $Tr_{new}:S_{\M}\times S_{\A}\times  A \rightarrow \pi(S_{\M})$ and $R_{new}:S_{\M}\times S_{\A}\times  A \rightarrow \mathbb{R}$,
  where $tr_{new}((s_{\M},s_{\A}),a,(s_{\M}',\delta(s_{\A},s')) =  tr_{old}(s_{\M},a,s_{\M}')$ (and all other entries are 0).
  In practice, because the $S_{\A}$ component is deterministic, there is no need to explicitly compute and store $Tr_{new}$ -- 
  it is representable using the $Tr_{old}$ and $\delta$.
  We set $r_{new}((s_{\M},s_{\A}),a) = r_{old}(s_{\M},a)$ if the reward for $s_{\M},a$ was Markovian, and
  $r_{new}((s_{\M},s_{\A}),a) = r$ if the reward for $s_{\M},a$ was non-Markovian, $s_{\A}\in F$, and $r$ was the NMR's value.
  All other entries are 0.
  At this stage, assuming $\A$ captures the right language, 
  there is no need to generate any new samples, because the error (due to sampling) is the same as before (because $\A$  is deterministic). 
 
  \rmax has an inherent optimism bias. 
  However, this bias induces state-action exploration, not trace exploration.
  To motivate more exploration we do the following.
  When we observe the result of $a$ in $s$ once,  we do not update $r(s,a)$ because
  it may be non-Markovian. Instead,  we try to observe it multiple times to
  generate good input for the automata-learning component, and to ensure that we do not under-estimate its value.
  In addition, in our experiments, we run an $\epsilon$-greedy version of \rmax\ to encourage additional exploration.

\section{Empirical Evaluation}
  We evaluated the algorithms on two environments: non-Markovian multi-armed bandit (MAB) and robot world.
  Non-Markovian MAB is the simplest domain on which we can experiment with learning NMRs. 
  While no state exploration is needed, we need to explore actions sequences to learn the NMR.
  Robot world is a grid-like domain augmented with a variety of NMRs.
  For automata learning we implemented a version of the \lstar\  algorithm in Python based on the GI-Learning library (\url{github.com/gabrer/gi-learning}) 
  and used the EDSM implementation from the FlexFringe library~\cite{FlexFringe} with its code changed to support trace weights.
  All  solvers uses $\epsilon$-greedy exploration with simulated annealing from 0.9 to 0.1 using a rate of 1e-6 updates each step.
  The learning rate $\alpha$ was set to 0.1.

\subsection{Multi-Armed Bandit}
  MAB provides the simplest RL setting. Agents get to choose among $N$ machines/arms/actions in each step.
  In standard MAB, an average reward is associated with each arm. In our domain the reward is deterministic, but it is non-Markovian.
  An agent that treats this problem as a standard, Markovian MAB will learn to associate an average reward with each arm, 
  and will not realize that this reward is related to its past behavior. 
  In our experiments we considered four different reward structure: (1) Reward for using arm 1 four consecutive steps, followed by arm 3.
  (2) Like reward 1, but received with a delay of (arbitrary) three steps. 
  (3) Play arm 3 twice in a row, then arm  2. % ($\Sigma^*332$).
  (4) Reward 1 and Reward 3.

  \begin{figure*}[t] 
    \centering
    \begin{subfigure}[b]{0.24\textwidth}
      \includegraphics[width=\textwidth]{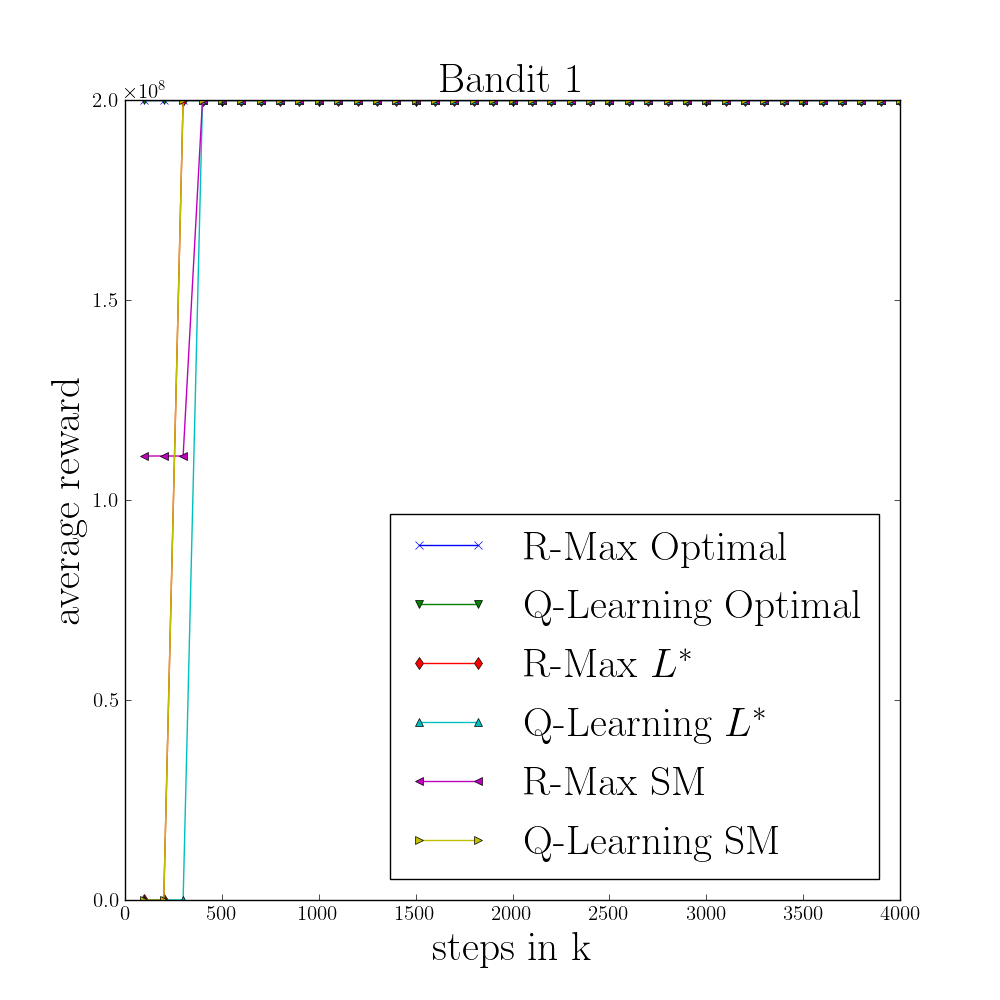}
    \end{subfigure}
    \begin{subfigure}[b]{0.24\textwidth}
      \includegraphics[width=\textwidth]{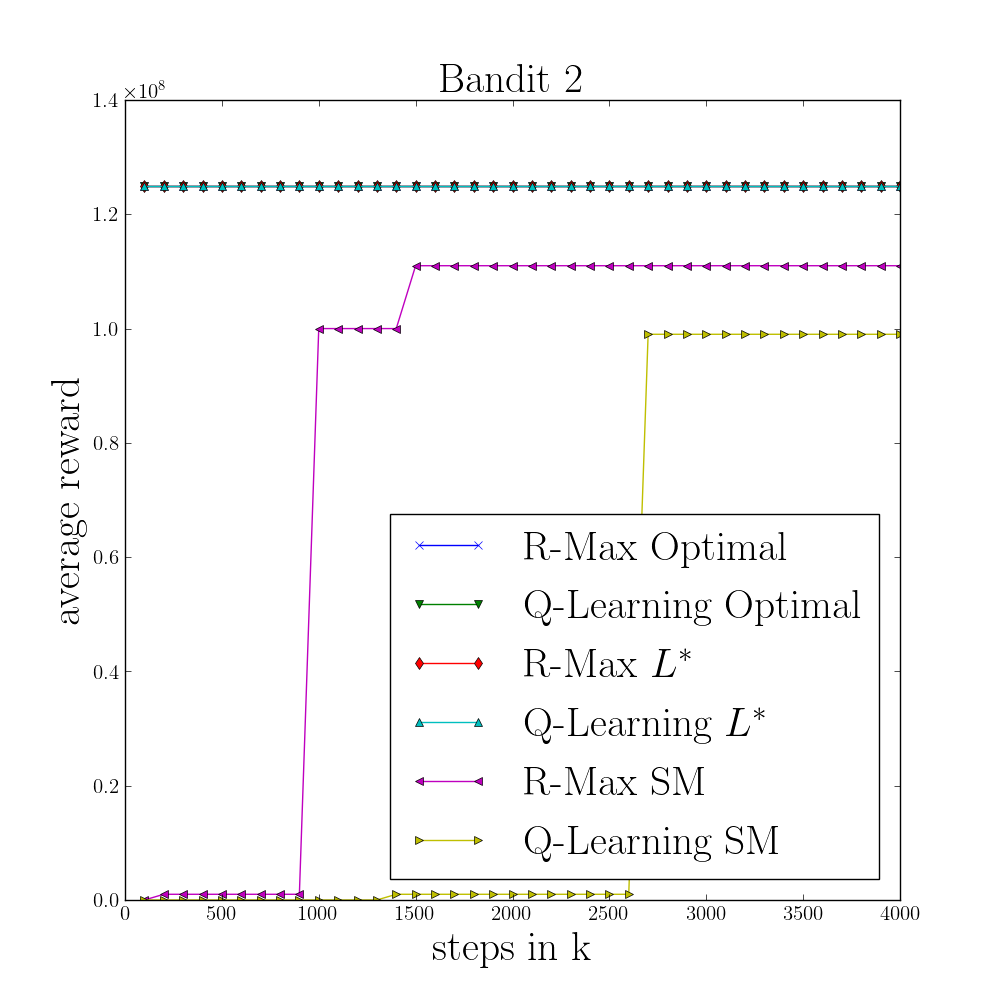}
    \end{subfigure}
    \begin{subfigure}[b]{0.24\textwidth}
      \includegraphics[width=\textwidth]{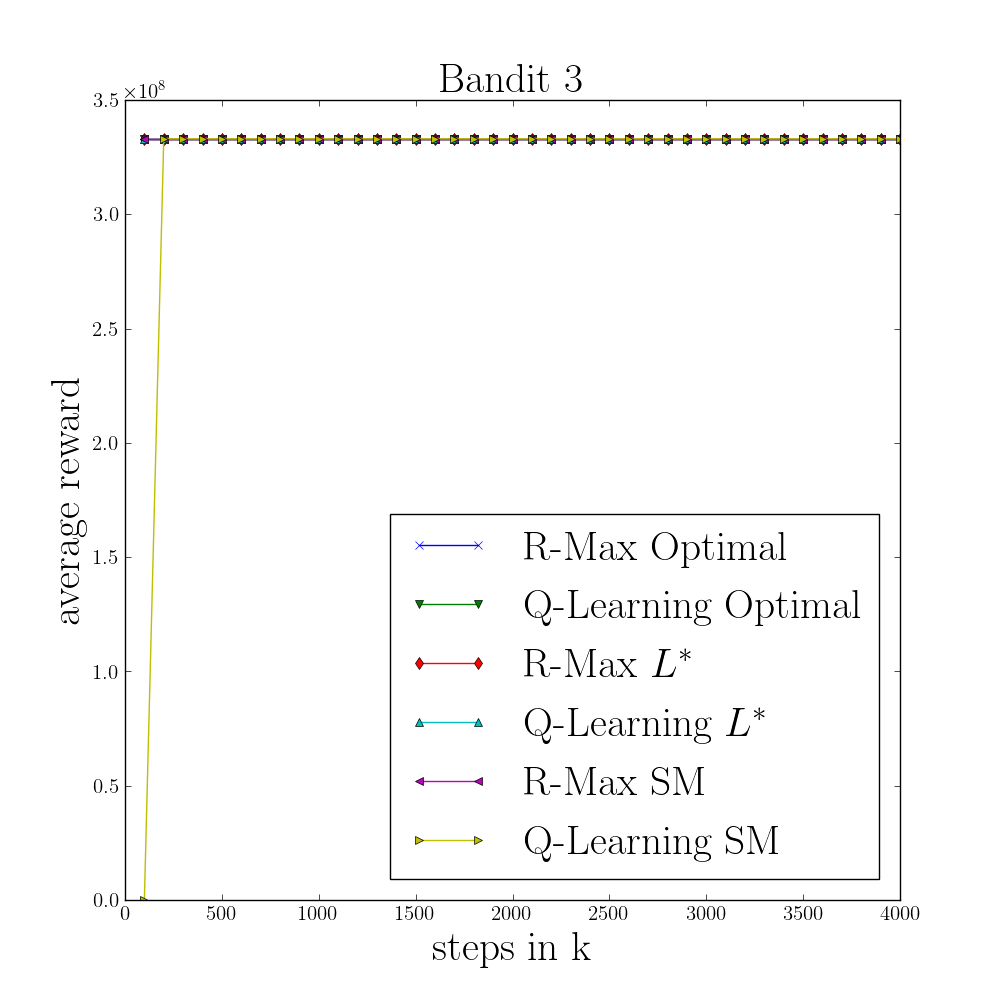}
    \end{subfigure}
    \begin{subfigure}[b]{0.24\textwidth}
      \includegraphics[width=\textwidth]{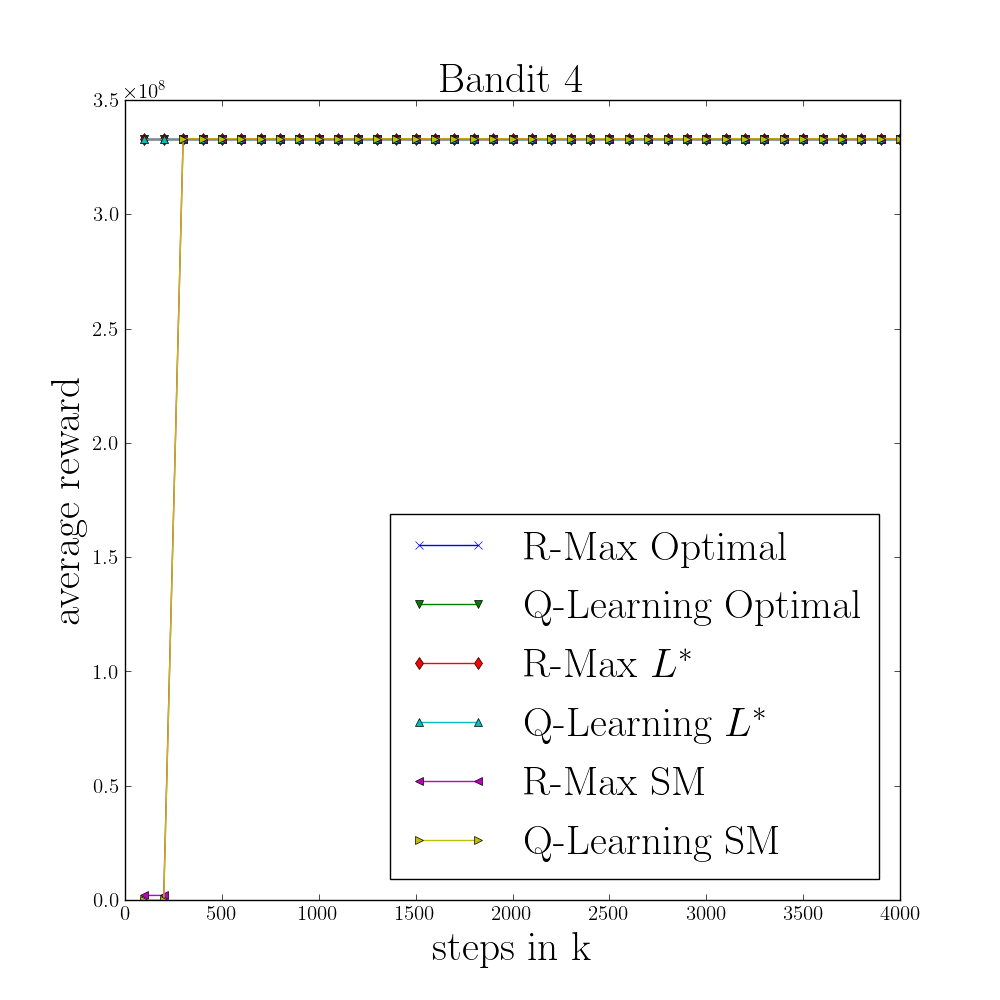}
    \end{subfigure}
    \caption{Non Markovian MAB Results}
    \label{fig:banditgraph}
  \end{figure*}
 
  The results of these experiments are shown in Figure~1.
  Each graph describes the average reward as a function of the number of steps of the learning algorithm. 
  Each episode consists of 20 steps. Tests were run for 4 million steps, with the results evaluated every 100,000 steps
  (= 5,000 traces).  The discount  factor was 0.99. 
  The value plotted for each time step is the average reward over 20 trials of the optimal policy at this point, i.e., 
  the greedy policy  w.r.t.~the current $q$ value or model parameters, with the best current automaton.

  \qlearn\ optimal and \rmax\ optimal refer to the result of running $Q-$learning and $R-$max when they are given the optimal automaton ahead of time. 
  In this case, the problem essentially reduces to learning in a regular MAB, 
  and the algorithms converge to the optimal policy before we reach the first test point (i.e., after less than 100,000 steps). 
  We did not plot the results of vanilla $Q-$learning and $R-$max because the greedy policy will select a single action only, 
  and this results in 0 reward.
 
  We also see the results of the four possible combinations of $Q-$learning and $R-$max with \lstar\ and \edsm. 
  \lstar\ performed much better than \edsm\ because we can provide a relatively efficient
  perfect teacher: Membership queries are easy to generate since there is a single state and the trace consists of actions only.
  Equivalence queries are answered the moment we see a mismatch between the prediction of the automaton and the current
  trace, and with reasonable exploration, this does not take too long.
  EDSM, too, will eventually learn the correct DFA, but it often goes through some non optimal DFAs, 
  containing too many states (and thus with worse generalization) or not capturing the NMR correctly. In these cases, its average reward is lower.

\subsection{Robot World}
  This is a 5x5 grid world, shown below, where the robot starts in a fixed position.
  \begin{center}
    {\includegraphics[width=.2\columnwidth]{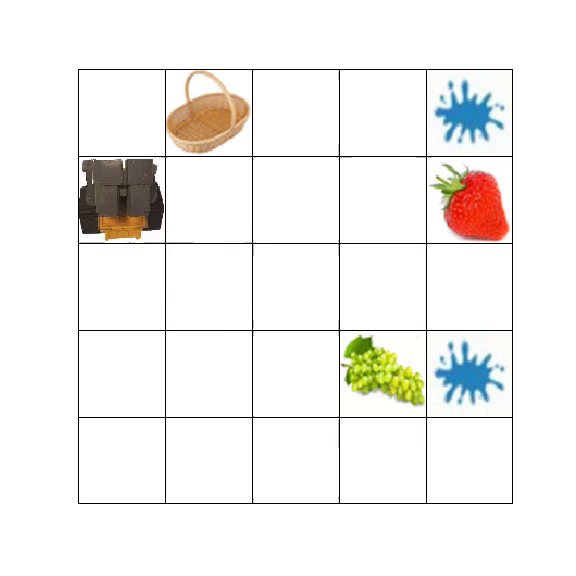}}
  \end{center}
  The grid contains a basket, in a fixed location, and two stains and two fruits in random locations. 
  The stains can be cleaned and the fruits can be picked and put in the basket. 
  The agent actions include moving {\em up/down/left/right, clean, pick}, and {\em put.}
  The robot's state consists of it $(x,y)$ position, the stain locations, the fruit locations, and the fruits held by the robot.
  Move actions have their intended effect with probability of 60\%. With probability 20\% each, 
  they cause a move to either the right or the left of the intended direction. 
  An attempt to move outside the grid results in remaining in place. 
  An attempt to clean where there is no stain or pick and put when there is no fruit has no effect.
  The pick, clean, and put actions can also fail with probability 40\%, in which case, the state does not change. 
  A constant Markovian reward of -1 is given for every {\em move} action and every illegal action (e.g., picking a fruit where none exists).
  Each trial ends after 60 steps, or if both stains are cleaned and the fruits are in the basket.

  The various NMRs were given for:
  (1) Picking a fruit, provided all stains were cleaned -- applicable for each pick.
  (2) Picking a fruit, provided all stains were cleaned, but the reward is delayed by 3 time steps -- applicable only for the first pick.
  (3) Picking a fruit, provided the previous two actions were {\em move-right}.
  (4) Both Reward 1 and 3.

  Each experiment consists of 25,000,000 steps. The policy was evaluated every 1,000,000 steps.
  The discount factor was 0.999999. 
  When learning the automaton, we ignore actions with no effect (e.g., moving outside the grid, picking a non-existent fruit).
  Policies are evaluated by taking their average discounted reward over 20 episodes.
  Each point on the graphs is an average of 11 runs with an error bar of +/- STD.

  \begin{figure*}[t] 
    \centering
    \begin{subfigure}[b]{0.24\textwidth}
      \includegraphics[width=\textwidth]{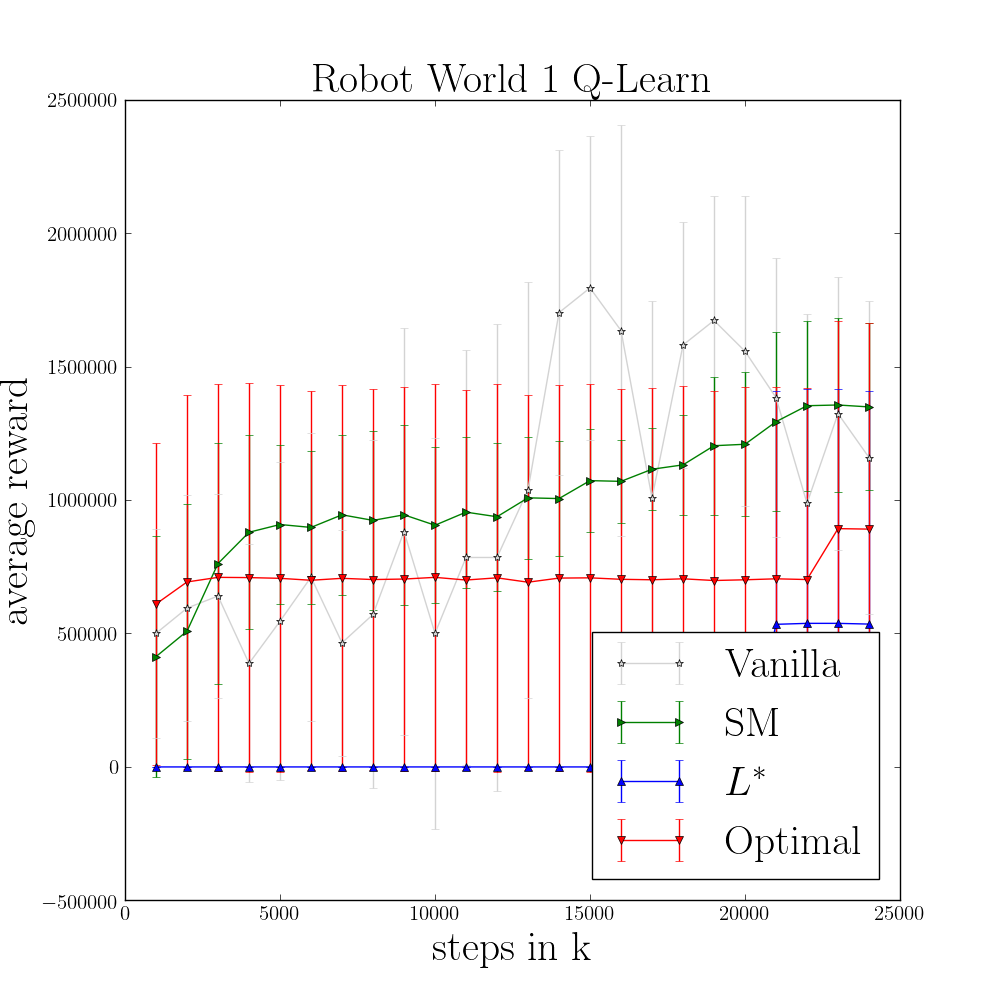}
    \end{subfigure}
    \begin{subfigure}[b]{0.24\textwidth}
      \includegraphics[width=\textwidth]{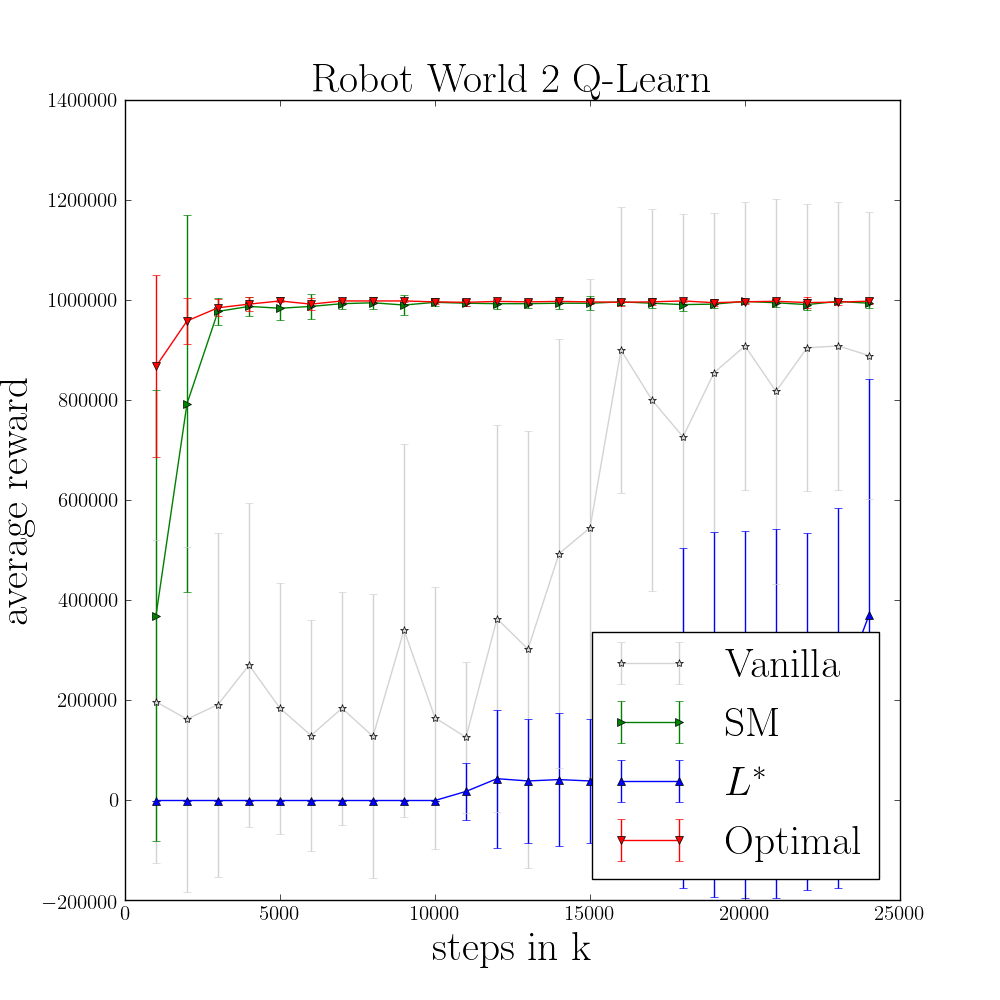}
    \end{subfigure}
    \begin{subfigure}[b]{0.24\textwidth}
      \includegraphics[width=\textwidth]{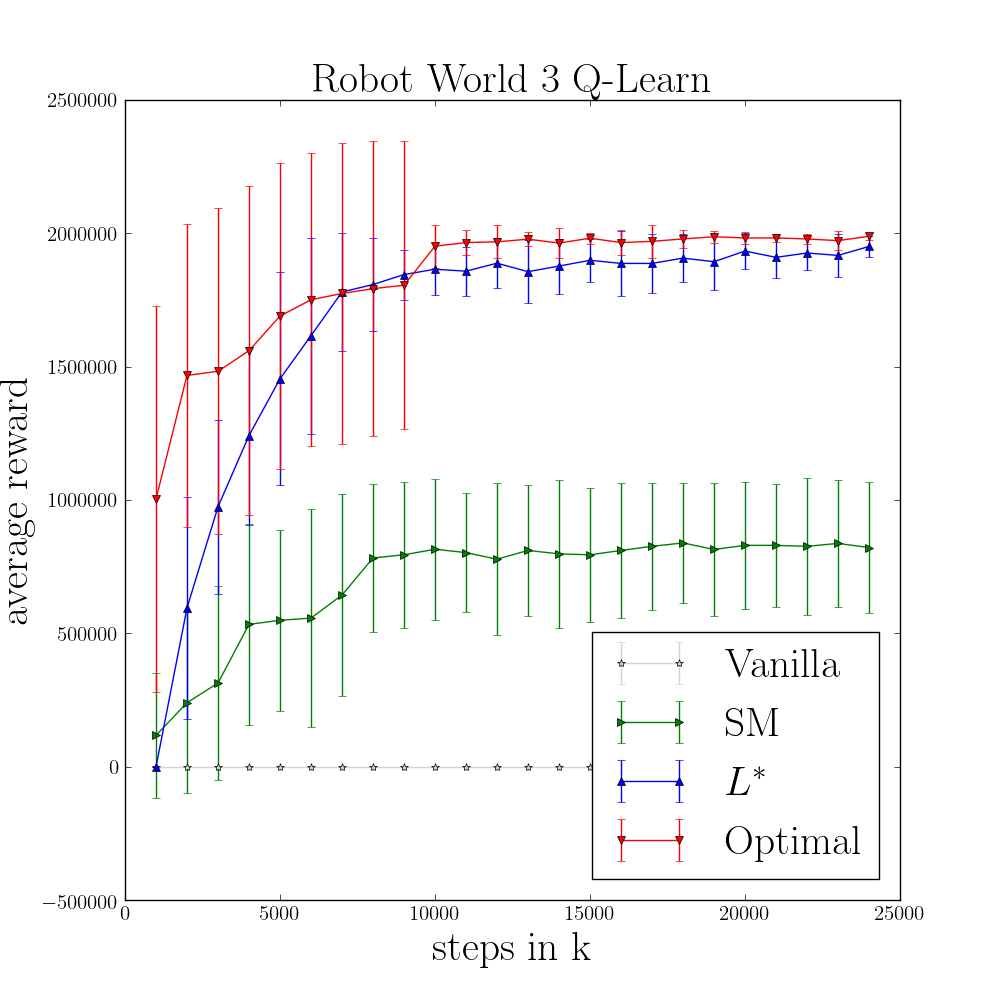}
    \end{subfigure}
    \begin{subfigure}[b]{0.24\textwidth}
      \includegraphics[width=\textwidth]{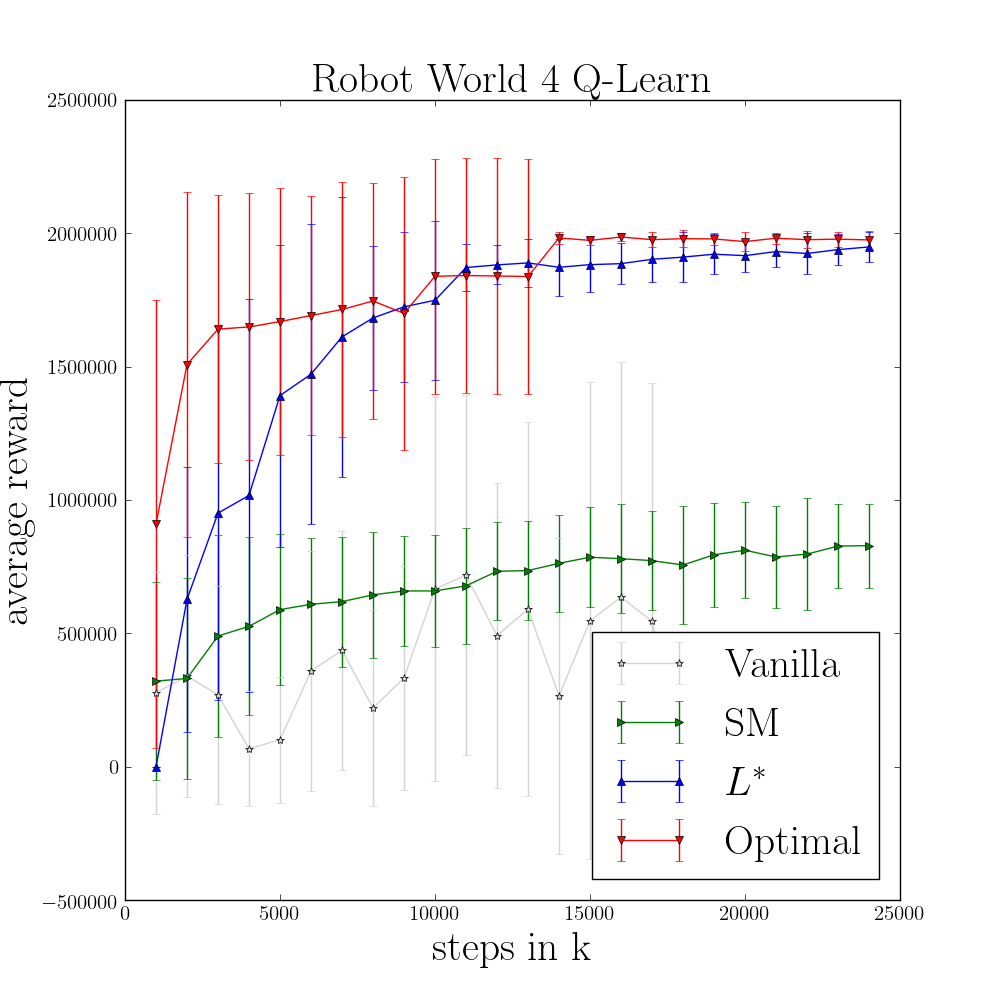}
    \end{subfigure}
    \begin{subfigure}[b]{0.24\textwidth}
      \includegraphics[width=\textwidth]{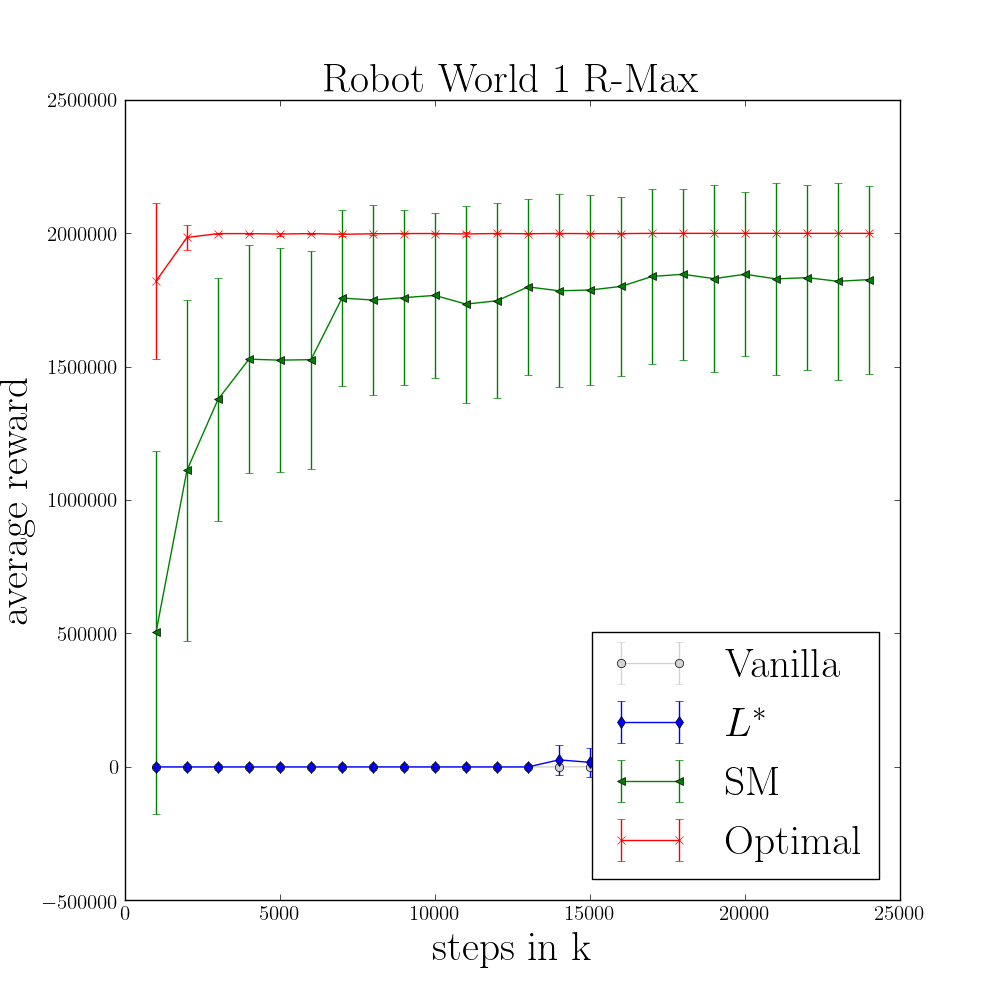}
    \end{subfigure}
    \begin{subfigure}[b]{0.24\textwidth}
      \includegraphics[width=\textwidth]{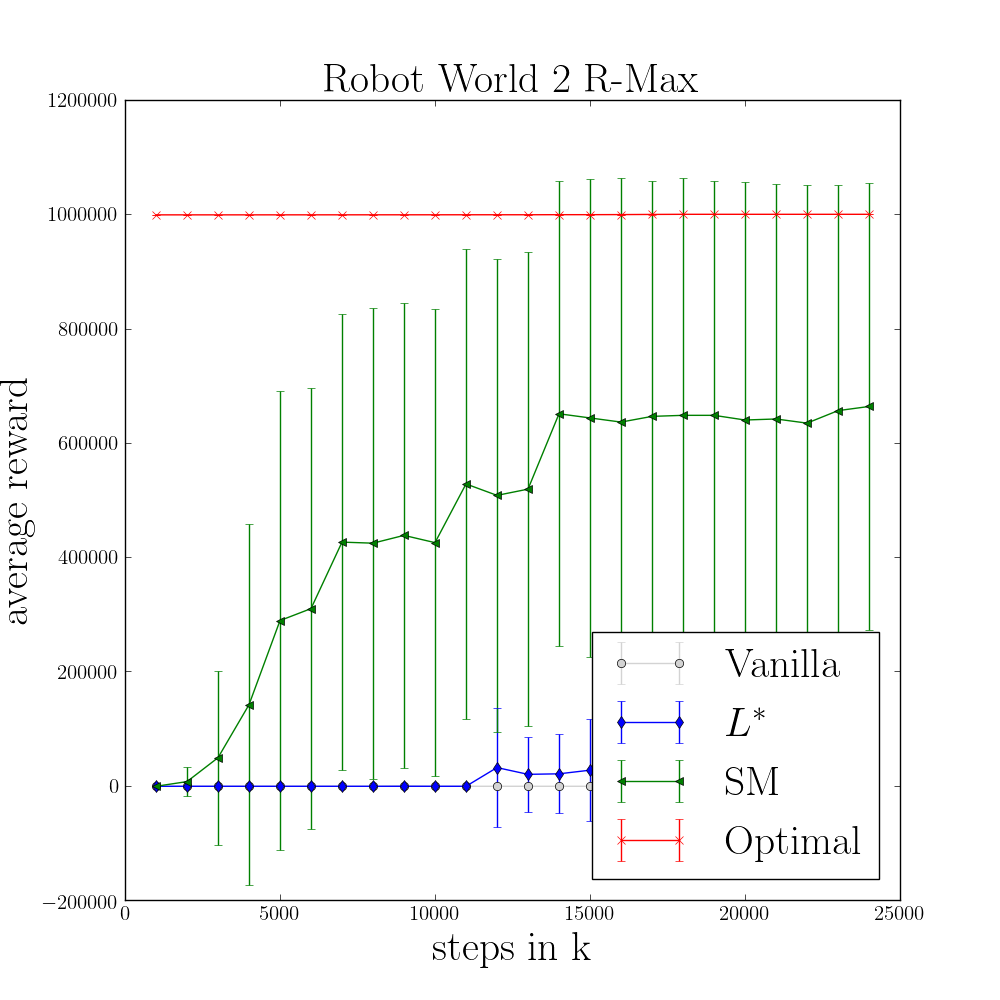}
    \end{subfigure}
    \begin{subfigure}[b]{0.24\textwidth}
      \includegraphics[width=\textwidth]{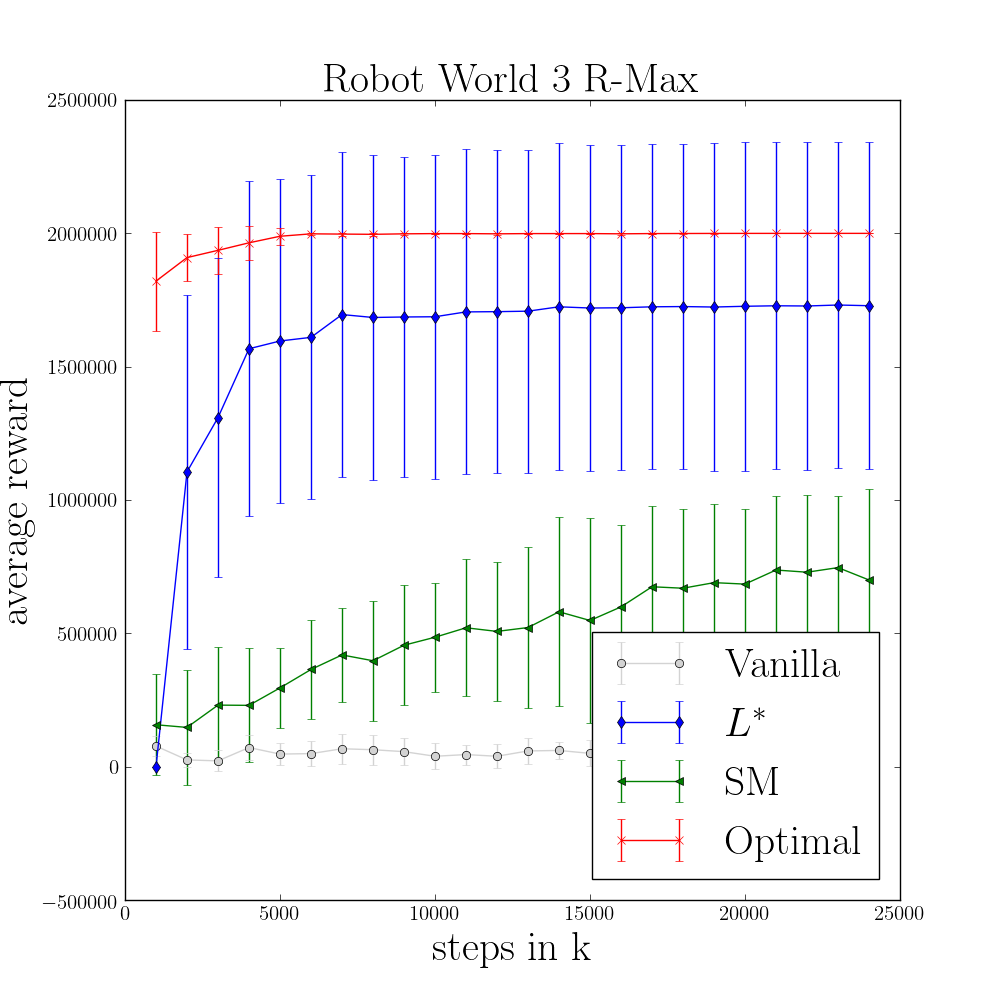}
    \end{subfigure}
    \begin{subfigure}[b]{0.24\textwidth}
      \includegraphics[width=\textwidth]{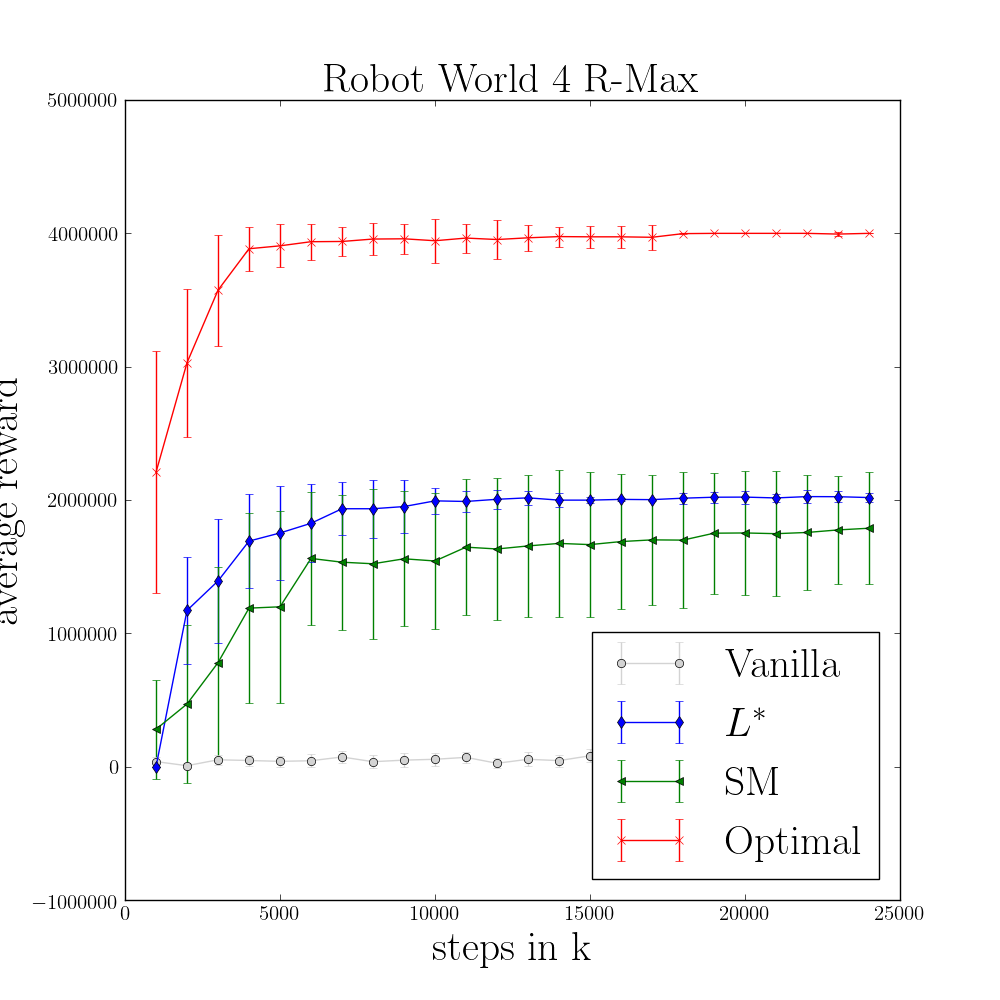}
    \end{subfigure}
    \caption{Non Markovian Robot World Results}
    \label{fig:robot_worldgraph}
  \end{figure*}
 
  Results appear in Figure~2. We plot the optimal $Q$-learning and $R$-max values (where the correct automata were given)
  as baselines. Both converge quickly to near-optimal behavior. We also plot
  the four variants of RL with \lstar\ and \edsm, and the results of 
  vanilla $Q$-learning and $R$-max, which assume that the domain is Markovian. 

  \rmax\ -- Reward Scheme 1: Both the optimal solver and \edsm\ find the optimal policy.
  The optimal solver takes roughly 2M steps while \edsm\ improves logarithmically and achieves sub-optimal results after roughly 8M steps.
  \lstar\ and vanilla fail to make significant progress. Specifically, \lstar\ fails due to the sparsity of observable traces, 
  hence resulting in an inferior teacher. 
  This results in a slowly learned big DFA that over-fitted to the specific map.
  
  \qlearn\ -- Reward Scheme 1: The performance is inferior to \rmax\, 
    Since each evaluated policy re-evaluate the last 50,000 saved traces collected via current best exploration policy,
    the result policy tends to be biased toward those sub-optimal observations.
  \edsm\ achieved a better result than the optimal DFA solver since it uses DFAs that partially encode the current grid object positions, 
  hence result with a slight better policy for that specific grid map.
  Vanilla \qlearn\ fails to find a good policy, and due to explorations and the stochastic dynamics, 
  fluctuates greatly between different evaluation points, although a general improving trend can be observed.
  
  \rmax\ -- Reward Scheme 2: vanilla \qlearn\  is very noisy and fails to converge. 
  The optimal solver finds the best policy roughly after 1M steps.
  \edsm\ improves gradually toward the optimal policy and \lstar\ still fails to converge.

   \qlearn\ -- Reward Scheme 2: Both the optimal solver and \edsm\ converge to the optimal policy roughly after 3M steps. 
  Vanilla is still noisy and converges to a sub-optimal policy after 25M steps with very high STD.
  \lstar\ shows some improvement but still perform poorly.
  
  Both -- Reward Scheme 3: Both the optimal solver and \lstar\ converge to the optimal policy roughly after 6M \& 10M steps.
  \edsm\ is stuck on a sub-optimal policy and finds a single NMR (slightly inferior to the one found by \qlearn). The vanilla algorithms
  fail to find a policy that achieves any NMR.
  
  \rmax\ -- Reward Scheme 4: Optimal finds the best policy after roughly 5M steps.
  Both \lstar and \edsm\ find sub-optimal policies with 2 NMRs after roughly 8M steps. 
  The vanilla policy fails to find a policy that achieve any NMR.
    
  \qlearn\ -- Reward Scheme 4: Both the optimal solver and \lstar\ converge to a sub-optimal policy, achieving half of the NMR's roughly after 15M steps.
  \edsm\ converges on a sub-optimal policy achieving 1 NMR and vanilla finds a policy with 1 NMR that highly depends on the stochastic behaviour of the world. 

  To ensure convergence in the above scenarios, we gradually limited trace length towards that of the minimal positive one
  because when traces are too long, the resulting DFAs are too large. 
  We also limited the number of negative traces to 1K. Since any prefix of a positive/negative trace is also negative, their
  number increases quickly. Thus, we maintain them in a FIFO structure.
  Finally, if the resulting DFA has more than 20 states, we consider this a failure. In that case, the maximal length of 
  the negative traces is reduce by 50\% and the algorithm is called again.
  This process continues until a DFA is inferred or a maximum of 20 failures.

\section{Discussion and Future Work}
  From the experiments we see that \lstar\  is better suited for learning short, simple reward models as in  MAB. Not
  surprisingly, \edsm, which does better in grammar learning,
  is better in the more realistic and challenging robot-world.
  Even when \edsm\  learns an approximate model of optimal DFA, this model still allows
  \rmax\ and \qlearn\  to converge on a decent sub-optimal policy. 

  Another observation is that the vanilla \rmax\ policy can essentially learn a fixed desirable trace if such
  a trace is representable in the policy. This cannot happen in MAB because there is a single state, and the optimal
  policy repeats the same action. But in Robot World, good traces do not require acting differently in the same state,
  and this allows the vanilla algorithm to perform better than the more sophisticated DFA learning algorithms on
  some scenarios. However, from preliminary additional tests conducted, this phenomenon becomes much less likely 
  as domain size increases. Interestingly, vanilla \qlearn\ usually has many fluctuations and is not able to converge to some policy.
  We are not sure why it does not learn a good path.

  The most interesting question for future work is how to induce good exploration of good traces to improve the
  input to the automata learning algorithm. 
  Better exploration is also likely needed to improve our ability to  learn multiple rewards, which was challenging for our algorithms. 
  It seems that once we find one reward, this  strongly biases our algorithms, and more focused exploration
  is probably needed to learn additional rewards. 
  Yet another reason to improve exploration is the large sample sizes required to learn. 
  Perhaps we need trace-oriented exploration. For example, a policy that with probability $\epsilon$ 
  performs a number of random trials, or something similar. 

  Another important direction for future work is attempting to better integrate the two learning tasks. 
  Right now, our exploration policy is not informed by the decisions that the automata learning algorithms must
  make (except when attempting to answer membership queries). For example, it may be possible to bias exploration when
  using \edsm\ to help \edsm\ make better merge decisions.

  Stochastic rewards are another challenge because it is difficult to differentiate between stochastic
  Markovian rewards and NMRs.  To identify them, we probably need more exploration. Then, we need to
  compare different hypothesis and see which one explains the data better.

  Finally, very recently,~\cite{CamachoM19} suggested learning \LTLf formulas directly, instead of learning the automaton. 
  This approach can be much more efficient if the reward corresponds to a small
  \LTLf\ formula, which is quite natural in many cases, and we believe adapting our algorithms by
  replacing the automata learning algorithms with this method is a promising direction for future work.

\section{Summary}
  We described a number of variants of algorithms that combine RL techniques with automata learning to learn in MDPs
  with NMRs. We proved that one such variant converges in the limit, and empirically evaluated all variants
  on two domains.  To the best of our knowledge, our work is the first to explore this problem. It highlights the many
  challenges that this set-up raises, and provides a solid starting point for addressing this problem and motivation for
  much future work.
  
\section{Acknowledgements}
  We thank the reviewers for their useful comments. This work was supported by ISF Grants 1651/19,
  by the Israel Ministry of Science and Technology Grant 54178, and by the Lynn and William Frankel Center for Computer Science.

\bibliographystyle{aaai}
\bibliography{AAAI-GaonM.8332}

\end{document}